\documentclass[twoside,11pt]{article} 

\usepackage{jmlr2e}
\usepackage{times}
\usepackage{amsmath}
\usepackage{mathrsfs}
\usepackage{scalerel}
\usepackage{textcomp}
\usepackage{caption}
\usepackage{subcaption}

\usepackage{thmtools}
\usepackage{thm-restate}

\newcommand{\objset}{\mathcal{X}}
\newcommand{\targetvalset}{\mathcal{Y}}
\newcommand{\spacefull}{\mathbb{X}}
\newcommand{\probl}{\rho}
\newcommand{\probx}{\rho_{\objset}}
\newcommand{\probxc}{\rho_{\spacefull}}

\newcommand{\probyx}{\rho_{\targetvalset|\spacefull}}

\newcommand{\SizeEdge}{n}

\newcommand{\expectwithbrackets}[2]{\mathbb{E}_{#1}{\left[#2\right]}}
\newcommand{\condexpectwithbrackets}[3]{\mathbb{E}_{#1}{\left[#2\middle| #3\right]}}

\newcommand\xqed[1]{%
  \leavevmode\unskip\penalty9999 \hbox{}\nobreak\hfill
  \quad\hbox{#1}}
\newcommand\eoe{\xqed{$\blacklozenge$}}
\newcommand\eof{\xqed{$\blacksquare$}}
\begin{document}

\title{Learning from Networked Examples\thanks{A main author of this paper is a student.}}


\author{\name{Yuyi Wang} %\email{yuwang@ethz.ch}
\\
\addr ETH Z\"{u}rich, Switzerland
\AND 
\name{Zheng-Chu Guo} %\email{guozhengchu@zju.edu.cn}
\\
\addr Zhejiang University, P.R.\ China
\AND
\name{Jan Ramon} %\email{jan.ramon@cs.kuleuven.be}
\\
\addr Inria, France \& KULeuven, Belgium}

\date{Received: date / Accepted: date}

\editor{The correct dates will be entered by the editor}

\maketitle

\begin{abstract}
Many machine learning algorithms are based on the assumption that training examples are drawn independently. 
However, this assumption does not hold anymore when learning from a networked sample because two or more training examples may share some common objects, 
and hence share the features of these shared objects. 
We show that the classic approach of ignoring this problem potentially can have a harmful effect on the accuracy of statistics, and then consider alternatives. 
One of these is to only use independent examples, discarding other information. 
However, this is clearly suboptimal. 
We analyze sample error bounds in this networked setting, providing significantly improved results. 
An important component of our approach is formed by efficient sample weighting schemes, which leads to novel concentration inequalities. 
%have a beneficial effect on concentration bound analysis. 
\end{abstract}


%Section 1: Introduction
\section{Introduction}\label{sec:intro}
Recently, there has been an increasing interest in network-structured data, 
such as data in social networks \citep{leinhardt2013social}, economic networks \citep{knoke2014economic}, citation networks \citep{nykl2014pagerank} and chemical interaction networks \citep{szklarczyk2014string}. 
An important challenge is that data concerning related objects cannot be assumed to be independent. 
More precisely, a key assumption made by many approaches in the field of statistics and machine learning is that observations are drawn independently and identically (i.i.d.)\ from a fixed distribution. 
However, this assumption does not hold anymore for observations extracted from a network. 

Many practical approaches to supervised learning in networks ignore (at least partially) the problem and learn models with classic machine learning techniques. 
While these work to some extent, they are not supported by a well-developed theory such as the one which provides generalization guarantees for the i.i.d.\ case 
as a function of the number of training examples. 
As a consequence, one may miss opportunities to learn due to the numerous dependencies between the training examples. 

In this paper, we make a step towards addressing this problem. 
More specifically, the contributions of this paper are as follows:
\begin{itemize}
\item We introduce a framework for explicitly representing dependencies among examples. Such framework allows for encoding domain knowledge and can form a basis for studies preparing a machine learning effort on networked data. 
\item We introduce a relaxed i.i.d.\ assumption for networked examples, which generalizes over classic i.i.d.\ assumptions and covers a number of existing learning problems. 
\item We show new concentration inequalities for networked examples and demonstrate the applicability of our results to learning theory by upgrading an empirical risk minimization result to networked data. In the process, we improve on earlier concentration inequalities by \cite{janson04} and then improve on earlier learning guarantees, even if we use the original unweighted ERM. 
\item We propose a novel learning scheme based on appropriate weighting on networked examples, which is efficiently implementable by solving a linear program and satisfies monotonicity which ensures that one can always do better when the networked dataset becomes larger (in the sense of super-hypergraphs). 
\end{itemize}

\subsection{Related work}\label{sec:relatedwork}
\cite{usunier06} define interdependent training data which is closely related to networked examples, make similar assumptions, 
and show generalization error bounds for classifiers trained with this type of data. 
In \citep{Liva10—JMLR}, PAC-Bayes bounds for classification with non-i.i.d data are established based on fractional colorings of graphs, results also hold for specific learning settings such as ranking and learning from stationary $\beta$- mixing distributions. 
\citep{Liva15—ICML} establishes new concentration inequalities based on the idea of fractional chromatic numbers and entropy methods for fractionally sub-additive and fractionally self-bounding functions of dependent variables. All these results are based on Janson's work \citep{janson04} and use dependency graphs \citep{erdos1975problems} to represent the examples and their relations. 
In this paper, we use a different representation for networked examples, improve the concentration inequalities in \citep{janson04} and then potentially improve on these existing learning results. Besides, we propose a better learning scheme than that in these existing works. 

\cite{Wang2011-ecml:proc} considered a similar setting of networked examples. 
Their work assumes a bounded covariance between pairs of examples connected with
an edge (excluding possible higher-order interactions) but this is not assumed in our model. While we use our model to show learning guarantees, 
\cite{Wang2011-ecml:proc} shows corrections for the bias (induced by the dependencies between examples) on statistical hypothesis tests. 

Mixing conditions are also used to model non-i.i.d.\ samples.
For example, in \citep{sun10,guo11}, the learning performance of regularized classification and regression algorithm using a non-i.i.d.\ sample is investigated, 
with the assumption that the training sample satisfies mixing conditions.
\cite{modha96} presented a Bernstein-type inequality for stationary exponentially $\alpha$-mixing processes based on the effective number (less than the sample size).
Our Bernstein-type inequalities for networked random variables assigns weights to examples. 
However, the assumptions for the training sample are different, and the main techniques are distinct. 
We refer interested readers to \citep{bradley05} and references therein for more details about the mixing conditions. 

Our theory is applicable to problems in the field of Statistical Relational Learning (SRL) \citep{Getoor07:other}, 
e.g., for learning local conditional probability functions for directed models such as Probabilistic Relational Models \citep{Friedman99:proc}, Logical Bayesian Networks \citep{Fierens05-ILP:proc}, 
Relational Bayesian Networks \citep{Jaeger97:proc}.  
There is a huge literature in SRL for learning features and existence of edges in a graph, for which we refer the reader to the survey by \cite{Rossi2012:jrnl}.  

%Section 2: Problem statement
\section{Model}\label{sec:problem}
In this section, we introduce networked examples. 
The basic intuition is that a networked example combines a number of objects, 
each of which can carry some information (features). 
Every object may be shared among several examples, such that we get a network. 
The sharing of information also makes explicit the dependencies between the examples. 

%Section 2.1: Networked examples
\subsection{Networked examples}
\label{sec:networked.examples}
We use a hypergraph $G=(V_G,E_G)$ to represent a set of networked examples.  
The vertices $V_G = \{v_1,v_2,\ldots,v_m\}$ represent objects, 
and the hyperedges $E_G = \{e_1,e_2,\ldots,e_n\}$ represent examples grouping a number of objects.  
In what follows, we will often abuse terminology, identify vertices and objects and use the terms interchangeably. 

We denote with $\objset$ the space of features of objects. 
These features can be represented with feature vectors or other suitable data structures. 
As usual in supervised learning, an example is a pair of an input and a target value.  
The input of an example is a collection of features of objects and is represented by a multiset of elements of $\objset$. 
We denote with $\spacefull$ %=\mathbb{N}^{\objset}$ 
the space of features of all possible examples. 
Each training example has a target value, e.g., a class label or regression value. 
We denote with $\targetvalset$ the space of all possible target values. 

This hypergraph-based representation together with the weaker independence assumption (see the next subsection) 
have powerful expressive ability and cover many existing models. 
For example, many previous works treat the ranking problem as binary classification on networked data. 

\paragraph{Learning to rank}
\textit{Consider the pairwise document ranking problem, which has been extensively studied for information retrieval \citep{liu2009learning}. 
We have a hypergraph $G$ with a vertex set $V_G$ of documents, and 
every hyperedge contains two document vertices (this hypergraph is actually a normal graph since every hyperedge has exactly two vertices). 
Every document is described by a feature vector from $\objset$ which could be a bag of words; every training example contains two documents (and their features from $\spacefull = \objset \times \objset$).  The target value of an example is the order between the two documents, i.e., $\targetvalset = \{0,1\}$ and $0$ means that the first document is ranked higher than the second one, $1$ otherwise. 
In fact, an even larger class of learning to rank problems can be modeled using networked examples (called \emph{interdependent examples} in the literature), and we refer the interested readers to \citep[Chap.\ 4]{amini2015learning}.} 
\eoe

In the following example, we make a simplified version of the problem of applying machine learning techniques to decide which projects to invest. 
In contrast to the example above, every hyperedge in the example of investment may have more than two vertices and the cardinality of every hyperedge may be different from each other. 

\paragraph{Investment}
\textit{Some organizations such as banks or research councils provide financial-related support, e.g., loans or grants, to business projects or research projects. 
These organizations may use machine learning techniques to help deciding which projects to invest. 
In the hypergraph-structured training dataset, every vertex is a person with a feature vector from $\objset$ describing education level, age, \ldots and every training example contains a group of people who take charge of a project. 
The target value of an example is the return on investment. 
It is common that some people may be involved in several different projects, and in this way we have a hypergraph. 
}
\eoe

Our model and results can also be applied to an important subclass of hypergraphs in which $V_G$ is a union of disjoint sets 
$\{V_G^{(i)}\}_{i=1}^k$ and $E_G\subseteq \times_{i=1}^k V_G^{(i)}$, i.e., $k$-partite hypergraphs. 

\paragraph{Movie rating}
\textit{Consider the problem of movie rating. 
We have a tripartite hypergraph $G$ with a vertex set $V_G^{(1)}$ of persons, a vertex set $V_G^{(2)}$ of movies and a vertex set $V_G^{(3)}$ of cinemas. 
Every hyperedge contains a person vertex, a movie vertex and a cinema vertex. 
The feature space $\objset$ also has three parts $\objset^{(1)}$, $\objset^{(2)}$ and $\objset^{(3)}$. 
Every person (a vertex in $V_G^{(1)}$) can be described by a feature vector from $\objset^{(1)}$ with gender, age, nationality, \ldots
Every movie (a vertex in $V^{(2)}$) can be described by a feature vector from $\objset^{(2)}$ with genre, actor popularity, \ldots 
Every cinema vertex (a vertex in $V^{(3)}$) can be described by a feature vector from $\objset^{(3)}$ with location, equipments, \ldots 
Then, $\spacefull=\objset^{(1)}\times \objset^{(2)}\times \objset^{(3)}$ is the space of feature vectors of complete examples, 
consisting of a concatenation of a person feature vector, a movie feature vector and a cinema feature vector. 
The target value of an example is the rating the person gave to the movie in the concerned cinema, e.g., the space $\targetvalset$ can be the set $\{1,2,\ldots,10\}$.
A trained model may help recommending movies and suggest good cinemas to watch the movie. 
} 
\eoe

%Section 2.2: Independence assumption
\subsection{Independence assumption}\label{sec:independence.assumption}
Though networked examples are not independent, 
we still need to assume some weaker form of independence of the examples. 
If we would not make any assumption, the dependence between examples could be 
so strong that they perfectly correlate. 
In such situation, it is not possible to learn. 

We consider a labeled hypergraph $(G,\objset,\phi,\targetvalset,\lambda)$, 
where the labels assigned by $\phi: V_G \mapsto \objset$ and $\lambda:E_G\mapsto \targetvalset$ are drawn randomly from a probability distribution $\probl$.  We make the following assumptions:
\begin{itemize}
\item Features of every object (assigned to vertices) are independent from features of other objects, 
i.e., there is a probability distribution $\probx$ on $\objset$ such that for every $q\in \objset$ and $v\in V_G$, 
$ \probl(\phi(v)=q \mid\phi(u)) = \probl(\phi(v)=q) = \probx(q) $ for any $u\neq v$. It is not required that objects have an identical distribution, though we use the same notation $\probx$ for all objects for the sake of simplicity. 

\item Moreover, every example (assigned to hyperedges) gets a target value drawn independently given the features of the objects (vertices) incident with the hyperedge, i.e., there is a probability distribution $\probyx$ on $\targetvalset\times \spacefull$ such that for all $e\in E_G$, $\probl(\lambda(e)=y \mid \phi|_e) = \probyx(y, \phi|_e) = \probl(\lambda(e)=y \mid \phi, E_G)$. Here, $\phi|_e$ is $\phi$ restricted to $e$, i.e., $\phi|_e=\{(v,\phi(v))\mid v\in e\}$. 
\end{itemize}

From the above assumptions, we can infer that 
\[\probl(\phi,\lambda) = \prod_{v\in V_G} \probx(\phi(v)) \prod_{e\in E_G} \probyx(\lambda(e), \phi|_e).\] 
Our analysis holds no matter what the distribution $\probl$ is, as long as the above assumptions are met. 

It is worth pointing out that this weak independence assumption may not yet hold in all real-world situations but is already a better approximation than the classic i.i.d.\ assumptions. 
% It may be instructive to consider real-world situations where our assumptions are satisfied and variants where they are not. 
For instance, consider again the problem of learning to rank, this weak independence assumption is  usually satisfied, and whether we know the target  value of a specific pair of documents only depends  on whether it is in the collected data, independent  of the features of the two documents. 
Besides, similar assumptions are also made in several previous works such as \citep{usunier06,Liva10—JMLR,Liva15—ICML}. 
% However, in our movie rating example, it may or may not be realistic
% that this assumption holds. In particular, if most visitors already have a preference and will not choose movies. On the other hand, if ratings are obtained during a sneak preview, experiment
% or movie contest where a number of participants or jury members are asked to watch a
% specific list of movies, one could randomize the movies to increase fairness. In this way our
% assumptions would be satisfied.

\section{Effective sample size, intuition, and examples}
An important aspect of the theory presented in this paper can be understood as a better estimation of the \emph{effective sample size} of networked datasets. 
% Several slightly different definitions exist, but g
Generally speaking one can define the effective sample size of a dataset $G$,  for a particular statistical task of estimating $T$, as the number of examples an i.i.d.\ dataset would need to allow for estimating $T$ as accurately as can be done with the networked dataset $G$. 
In this paper, %to discuss the effective sample size, 
we consider \emph{mean value estimation} and use concentration inequalities as references. 

Classical concentration inequalities such as Bernstein inequality \citep{bernstein24}, Bennett inequality \citep{bennett62} and Chernoff-Hoeffding inequalities \citep{chernoff52,hoeffding63} (see e.g., \citep[Chap.\ 2]{chung06}) are used to analyze generalization errors of learning problems on i.i.d.\ samples. 

\cite{janson04} showed concentration inequalities (see Appendix \ref{ssec:eqw})  for the (equally weighted) mean value estimator of \emph{networked random variables} by the so-called fractional-coloring-based method. 

\begin{definition}[Networked random variables] Given a hypergraph $G$, we call $\left\{\xi_i\right\}_{i=1}^{\SizeEdge}$ \emph{$G$-networked random variables} if there exist functions $f_i:\spacefull\mapsto \mathbb{R}$ such that $\xi_i = f_i(\{\Phi_v\mid v\in e_i\})$ where $\{\Phi_v\}_{v\in V}$ is a set of independent random variables indexed by the vertices of $G$ and $\{\Phi_v\mid v\in e_i\} = x_i$ is the feature of example $i$. 
\end{definition}

In Janson's inequalities, a key parameter is the fractional hyperedge-chromatic number\footnote{In \citep{janson04}, this parameter is called the fractional cover number.} $\chi^*_{G}$. 

\begin{definition}[Fractional hyperedge-chromatic number]
A $b$-fold hyperedge-coloring of a hypergraph $G$ is an assignment of $b$ colors to every hyperedge in $E_G$ such that adjacent hyperedges have no color in common.
The $b$-fold hyperedge-chromatic number $\chi^{(b)}_G$ of a hypergraph $G$ is the smallest number of colors needed to obtain a $b$-fold hyperedge-coloring of the hyperedges in $E_G$.
The fractional hyperedge-chromatic number $\chi^*_G$ is 
$$\chi^*_G = \lim_{b\rightarrow\infty}\frac{\chi^{(b)}_G}{b} = \inf_b \frac{\chi^{(b)}_G}{b}.$$
\end{definition}

Though Janson's result and its variations are the foundation of (almost) all existing theoretical work related to learning from networked examples \citep{usunier06,Liva10—JMLR,amini2015learning}, we begin our technical discussion with the following examples which show that it is necessary and possible to improve Janson's result (also see the remark after Theorem \ref{thm:janson04}). 
We denote the effective sample size with Janson's method by $\textsc{CEss}$, i.e., $\textsc{CEss}(G) = n/\chi^*_G$, and denote the effective sample size with our new method by $\textsc{MEss}$. According to the results in Sec.\ \ref{sec:concentration.inequalities}, $\textsc{MEss}(G) = \nu^*_G$ the fractional matching number of $G$ (in Example \ref{example:details} and \ref{ex:oddcycle}, $\textsc{MEss}(G) = \nu^*_G = n/\omega_G$ where $\omega_G$ is the maximum degree of $G$). %which equals to the fractional matching number $\nu^*_G$ (see ). 

%example: triangles
\begin{example}[%Details matter - 
$\textsc{MEss} = \Omega(n) \cdot \textsc{CEss}$]\label{example:details}
Consider two hypergraphs each of which has $n$ hyperedges: % (see Fig. \ref{fig:tri}):
\begin{itemize}
\item $G_n$: all the $n$ hyperedges share a common vertex; and
\item $H_n$: all the $n$ hyperedges intersect with each other, but there do not exist any three hyperedges that share a common vertex. 
\end{itemize}

% \captionsetup[subfigure]{labelfont=it,textfont=it}

% For example, the two hypergraphs of $n=3$ can be visualized as follows. 
% \begin{figure}[h]
% \centering
%  \begin{subfigure}[b]{0.3\textwidth}
%   \includegraphics[scale=0.25]{./tri_1.eps}
% %   \caption[]{}
%   \label{subfig:tri_1}
%  \end{subfigure}
%  \qquad
%  \begin{subfigure}[b]{0.3\textwidth}
%   \includegraphics[scale=0.25]{./tri_2.eps}
% %   \caption[]{}
%   \label{subfig:tri_2}
%  \end{subfigure}
%  %\caption{Two hypergraphs have the same coloring } 
%  \label{fig:tri}
% \end{figure}

For any $n$, hypergraphs $G_n$ and $H_n$ have the same fractional hyperedge-chromatic number $\chi_{G_n}^* = \chi_{H_n}^* = n$ because in both cases any two hyperedges intersect with each other. 
According to Janson's result, both of them have effective sample size $n/\chi_{G_n}^* = n/\chi_{H_n}^* = 1$. 

For hypergraph $G_n$, this result is reasonable because all the networked random variables may be highly controlled by the shared vertex. 
However, if two networked random variables in $H_n$ are highly correlated, then intuitively the two variables cannot be very correlated with other networked random variables in $H_n$; hence the effective sample size of $H_n$ would be larger than $1$. 
This is indeed the case, as we show in Corollary \ref{cor:sameweight} the effective sample size of $H_n$ can be $n/2$ which cannot be obtained by fractional-coloring-based approaches %cannot let the effective sample size of $H_n$ go beyond $1$
because these methods only consider which examples are dependent, but not the details - how they are dependent. \textbf{Details matter! }
\end{example}

The example above shows, that if there is a subset of networked random variables  that any two of them overlap, then fractional-coloring-based approaches only deal with the worst case that all the networked random variables share a common vertex but ignore the details of the hypergraph. A natural question to ask is whether this is the ``only way'' that $\textsc{MEss}$ can beat $\textsc{CEss}$. In particular, are there other cases without ``the issue of details'' that $\textsc{MEss}$ can still be larger than $\textsc{CEss}$? 
The following example shows that the answer is yes. 

\begin{example}[$\textsc{MEss} > \textsc{CEss}$ on odd cycles]\label{ex:oddcycle}
Consider an odd cycle $G$ with $n=2k+1$ hyepredges, i.e., hyperedge $e_i$ only intersect with two hyperedges $e_{(i-1)\mod n} $ and $e_{(i+1) \mod n}$. 
The fractional hyperedge-chromatic number of $G_n$ is $2+1/k$ (see e.g.,\citep{scheinerman2011fractional}), so $\textsc{CEss}(G_n) = n/(2+1/k) = k$. 
By Corollary \ref{cor:sameweight}, we improve the result to $\textsc{MEss}(G_n) = n/2 = k + 1/2$. 
Note that in this example, there does not exist the issue of details mentioned in Example \ref{example:details}, but a gap between $\textsc{MEss}$ and $\textsc{CEss}$ still exists. 
\end{example}

In both examples above all the hypergraphs we considered have \emph{symmetric} structures, the equally weighted estimator seems work well. But, the following example shows that this simple estimator is not always good, as it does not satisfy %have an important property called
\emph{monotonicity}:  
Roughly speaking, given $n$ i.i.d.\ examples, when the sample size $n$ is larger then we would expect a more accurate estimator. Similarly, if $G'$ is a super-hypergraph of $G$, then the effective sample size of $G'$-networked examples should be not smaller than that of $G$-networked examples. 
The intuition is that if we have more data then we should learn better (at least not worse), or else we just select a subsample to learn. 

\begin{example}[Monotonicity]
Consider the following hypergraphs:
\begin{itemize}
\item $G$: has $n_1+n_2$ hyperedges in total, in which $n_1$ hyperedges are independent and the rest $n_2$ hyperedges share a common vertex. 
\item $H$: is a super-hypergraph of $G$, and it has $n_1+n_2+1$ hyperedges in total, in which $n_1$ hyperedges are independent and the rest $n_2+1$ hyperedges share a common vertex. 
\end{itemize}
The fractional hyperedge-chromatic number of $G$ and $H$ are $n_2$ and $n_2+1$ respectively (assume $n_1, n_2\ge 1$). Therefore, $\textsc{CEss}(G) = (n_1+n_2) / n_2$ is greater than $\textsc{CEss}(H) = (n_1+n_2+1) / (n_2+1)$ if we use this simple estimator, though $H$ is a super-hypergraph of $G$. 
\end{example}

In the example above, why don't we just ignore one of the $(n_2+1)$ hyperedges in $H$ and then obtain a better effective sample size that is the same as $\textsc{CEss}(G)$? 
This idea can be generalized: We just select a \emph{maximum matching} from the hypergraph, and then only use the examples corresponding to this matching. 
It is easy to see that
\begin{itemize} 
\item this method satisfies the monotonicity since the matching number $\nu$ of a hypergraph is always not larger than that of its super-hypergraph; 
\item besides, it is well known that $\nu_G \ge n/\chi^*_G$ for all hypergraphs, so it seems that the matching-based method can always achieve larger effective sample size than the fractional-coloring-based method. 
\end{itemize}

However, the matching-based method is still a suboptimal solution, because
\begin{itemize}
\item it is NP-hard to compute the maximum matching of a hypergraph in general,
\item though $\nu_G \ge n/\chi^*_G$ for all hypergraphs, we show out that $n/\chi^*_G$ is not tight for the equally weighted estimator, so matching-based method is not always better than the equally weighted estimator. 
\end{itemize}

In order to tackle this problem and accurately estimate the mean value\footnote{For the sake of simplicity, we assume every $\xi_i$ has the same expected value, and our analysis can be generalized to the case that $\xi_i$'s have non-identical expected values.} $\mu = \expectwithbrackets{}{\xi}$ using networked examples, we consider a class of estimators of the form
\[{\hat{\mu}}_{f,w}=\frac{\sum_{i=1}^{\SizeEdge} w_i \xi_i}{\sum_{i=1}^{\SizeEdge} w_i}\]
and three weighting schemes
\begin{itemize}
\item EQW: all examples get equal weights, i.e., $w_i=1$ for all $i$. As far as we know, almost all related works only deal with estimators of this type. 
\item IND: a maximum-size set $E_\mathit{IND}\subseteq E$ of independent examples is selected, i.e., $\forall e_1,e_2\in E_\mathit{IND}: e_1\cap e_2=\emptyset$. 
Examples in $E_\mathit{IND}$ are weighted $1$, otherwise weighted $0$. 
\item FMN: weight the examples with a fractional matching; it is introduced in Section \ref{ssec:networked.concentration}. 
\end{itemize}

%Section 3: Networked statistics and concentration bounds
\section{New concentration inequalities}\label{sec:concentration.inequalities}
\newcommand{\mathvec}[1]{\mathbf{{#1}}}

\newcommand{\targetfunc}{f}
In this section, we show new concentration bounds if random variables are allowed to have non-identical weights. 
These concentration inequalities cannot be obtained by Janson's method (the fractional-coloring-based method), so we prove new lemmas. 
In the next section, we use these results to bound generalization errors for learning from networked examples. 

%Section 3.2: Fractional Matching Schemes
\subsection{Fractional matching schemes}\label{ssec:networked.concentration}

Before stating the main result (Theorem \ref{thm:three.avg}), 
we first define fractional matchings. 

\begin{definition}[Fractional matching]
Given a hypergraph $G=(V_G,E_G)$ with $E_G=\{e_i\}_{i=1}^{\SizeEdge}$, a fractional matching $w$ is a nonnegative vector $(w_i)_{i=1}^n$ defined on the 
hyperedges satisfying that for every vertex $v\in V_G$, 
$\sum_{i:v\in e_i} w_i \le 1.$
In other words, a weight vector is a fractional matching if for every vertex
the sum of the weights of the incident hyperedges is at most $1$.
\end{definition}

\noindent\textbf{Remark:} The IND weighting scheme is also a (fractional) matching method and every weight $w_i$ must be either $0$ or $1$, i.e., these vectors correspond to matchings in hypergraphs. 
The original concentration inequalities can be applied directly to the IND weighting scheme, since all the examples with weight $1$ are mutually independent. Besides, since larger matchings provide tighter bounds, we tend to find a maximum matching whose sum is called the \emph{matching number} and is denoted by $\nu_G$. 
However, it is in general an NP-hard problem to find a maximum matching in hypergraphs (see e.g., \citep{garey79}).
Moreover, the maximum matching problem is also an APX-complete problem \citep{uriel91},
so we would not expect an efficient algorithm to achieve a good approximation either.

A key property used for proving classical concentration inequalities is that all observations are independent. 
That is, if $\left\{\xi_i\right\}_{i=1}^{\SizeEdge}$ are independent random variables, 
then the moment-generating function $\expectwithbrackets{}{\exp\left(c\sum_{i=1}^{\SizeEdge} \xi_i \right)}$, where $c\in \mathbb{R}$, satisfies
\[\expectwithbrackets{}{\exp \left(c\sum_{i=1}^{\SizeEdge} \xi_i\right)} = \prod_{i=1}^{\SizeEdge} \expectwithbrackets{}{e^{c\xi_i}}.\]
However, when considering networked random variables, the equality does not hold anymore. 
Instead, we show a new property in Lemma \ref{lemma:exponential}. 
\begin{lemma}\label{lemma:exponential}
Given $G$-networked random variables $\left\{\xi_i\right\}_{i=1}^{\SizeEdge}$,
if $w = (w_i)_{i=1}^{\SizeEdge}$ is a fractional matching of the hypergraph $G$,
then 
\begin{equation}\label{eq:exponential_ineq}
\expectwithbrackets{}{\exp\left(\sum_{i=1}^{\SizeEdge}{w_i \xi_i}\right)} \le \prod_{i=1}^{\SizeEdge} \left(\expectwithbrackets{}{e^{\xi_i}}\right)^{w_i}. 
\end{equation}
\end{lemma}
 \begin{proof}
 First, note that the expectation in the left hand side of Inequality\ \eqref{eq:exponential_ineq} is over the (independent) features 
 $S_1, \ldots, S_{|V_G|}$ 
 of the vertices of $G$, because these are the basic random variables of which the $\left(\xi_i\right)_{i=1}^{\SizeEdge}$ are composed.
 We prove this theorem by induction on $|V_G|$.
 For $|V_G|=1$,
 \[\expectwithbrackets{}{\exp\left(\sum_{i=1}^{\SizeEdge}{w_i \xi_i}\right)} = \expectwithbrackets{S_1}{\prod_{i=1}^{\SizeEdge} e^{w_i \xi_i}}.\] 
 Using Lemma \ref{lemma:concave} with 
 $t = \left(e^{\xi_i}\right)_{i=1}^{\SizeEdge}$, 
 $\beta = w$ and $g(t) = \prod_{i=1}^{\SizeEdge}  e^{w_i \xi_i},$
 we know that $g(t)$ is a concave function since $w$ is a fractional matching.  
 Given that $g(t)$ is concave, we have 
 $$\expectwithbrackets{}{\exp\left(\sum_{i=1}^{\SizeEdge}{w_i \xi_i}\right)} = \expectwithbrackets{S_1}{g(t)} 
 \le  g(\expectwithbrackets{S_1}{t}) = \prod_{i=1}^{\SizeEdge} \left(\expectwithbrackets{}{e^{\xi_i}}\right)^{w_i}$$
 which follows from Jensen's inequality \citep{jensen}.
 Assume that the theorem is true for $|V_G| = 1,\ldots,m-1$,
 we now prove the theorem for $|V_G|=m$.  
 We can write 
 \begin{equation}\label{eq:mutual_independence}
  \expectwithbrackets{}{\exp\left(\sum_{i=1}^{\SizeEdge}{w_i \xi_i}\right)}
 = \expectwithbrackets{S_m}{\condexpectwithbrackets{}{\prod_{i=1}^{\SizeEdge} e^{w_i \xi_i}}{S_m}}.
 \end{equation}
 where the $\condexpectwithbrackets{}{\cdot}{\cdot}$ notation on the right hand side denotes a conditional expectation.
 We use the induction hypothesis on the right hand side of Eq.\ \eqref{eq:mutual_independence}, yielding 
 \begin{equation}\label{eq:induction_hypothesis}
  \expectwithbrackets{S_m}{\condexpectwithbrackets{}{\prod_{i=1}^{\SizeEdge} e^{w_i \xi_i}}{S_m}} 
 \le \expectwithbrackets{S_m}{\prod_{i=1}^{\SizeEdge} \left(\condexpectwithbrackets{}{e^{\xi_i}}{S_m}\right)^{w_i}}.
 \end{equation}
 We define two index sets $A$ and $B$, partitioning hyperedges in $G$ (and hence random variables $\xi_i$) into a part which is incident with $v_m$ (dependent on $S_m$) 
 and a part which is not, i.e.,  
 $A := \{i | v_m \in e_i\}$ and $B := \{i | v_m \notin e_i \}$. 
 Then, for all $i \in B,$ $\xi_i$ is independent of $S_m.$ 
 We can write this as
 \begin{equation}\label{eq:A_B}
  \expectwithbrackets{S_m}{\prod_{i=1}^{\SizeEdge} \left(\condexpectwithbrackets{}{e^{\xi_i}}{S_m}\right)^{w_i}}
 = \expectwithbrackets{S_m}{\prod_{i\in A} \left(\condexpectwithbrackets{}{e^{\xi_i}}{S_m}\right)^{w_i}}
 \prod_{i\in B} \left(\expectwithbrackets{}{ e^{\xi_i} }\right)^{w_i}.
 \end{equation}
 Let $t = \left(\condexpectwithbrackets{}{e^{\xi_i}}{S_m}\right)_{i\in A}$, 
 $\beta = (w_i)_{i\in A}$ and $g(t) = \prod_{i\in A} \left(\condexpectwithbrackets{}{ e^{\xi_i}}{S_m}\right)^{w_i}.$
 According to the definition of fractional matchings and Lemma \ref{lemma:concave}, we know that $g(t)$ is concave.
 Again, by Jensen's inequality, we have
 \begin{equation}\label{eq:concave}
  \expectwithbrackets{S_m}{\prod_{i\in A} \left(\condexpectwithbrackets{}{ e^{\xi_i}}{S_m}\right)^{w_i}} \le 
 \prod_{i\in A} \left(\expectwithbrackets{S_m}{\condexpectwithbrackets{}{e^{\xi_i}}{S_m}}\right)^{w_i} = 
 \prod_{i\in A} \left(\expectwithbrackets{}{e^{\xi_i}}\right)^{w_i}.
 \end{equation}
 From Equations\ \eqref{eq:mutual_independence}, \eqref{eq:A_B} and Inequalities\ \eqref{eq:induction_hypothesis} and \eqref{eq:concave}, 
 we can see that this theorem holds for $|V_G|=m.$
 \end{proof}
\noindent\textbf{Remark:} Similar results hold for all nonnegative functions of $\xi$, not only $e^{\xi}.$ More precisely, given $G$-networked random variables $\left\{\xi_i\right\}_{i=1}^{\SizeEdge}$, it holds that $\expectwithbrackets{}{ \prod_{i=1}^{\SizeEdge}\left(f(\xi_i)\right)^{w_i}} \le \prod_{i=1}^{\SizeEdge} \left(\expectwithbrackets{}{f(\xi_i)}\right)^{w_i}$ for any nonnegative function $f$ and any fractional matching $w$ of $G$. 

The follow result, which is the networked analogues of the Bennett inequality, shows how to use Lemma \ref{lemma:exponential} to prove new concentration inequalities.  
% and will be used to show generalization error bounds for learning from networked examples. 

\begin{lemma}\label{lemma:bennett}
Let $\left(\xi_i\right)_{i=1}^{\SizeEdge}$ be $G$-networked random variables with mean $\expectwithbrackets{}{\xi_i} = {\mu}$ and variance $\sigma^2(\xi_i) = \sigma^2$,
such that $|\xi_i-\mu|\le M$ with probability $1$. 
Let $w = (w_i)_{i=1}^{\SizeEdge}$ be a fractional matching for $G$, and let $|w| = \sum_i w_i$, then for all $\epsilon > 0$, 
\[\Pr\left(\sum_i w_i\left(\xi_i-{\mu}\right) \ge \epsilon \right) \le
\exp \left(-\frac{|w|\sigma^2}{M^2} h\left(\frac{M\epsilon}{|w|\sigma^2}\right)\right)\]
where $h(a) =(1+a)\log(1+a)-a$ for any real number $a$. 
\end{lemma}
\begin{proof}
Without loss of generality, we assume ${\mu} = 0$.
Let $c$ be an arbitrary positive constant which will be determined later.
Then
$$I := \Pr\left(\sum_{i=1}^{\SizeEdge} w_i \xi_i \ge \epsilon \right) = \Pr\left(\exp\left(c\sum_{i=1}^{\SizeEdge} w_i \xi_i\right) \ge e^{c\epsilon} \right).$$
By Markov's inequality and Lemma \ref{lemma:exponential}, we have
$$I \le e^{-c\epsilon}\expectwithbrackets{}{\exp\left(c\sum_{i=1}^{\SizeEdge} w_i \xi_i\right)} \le  e^{-c\epsilon} \prod_i \left(\expectwithbrackets{}{e^{c \xi_i}}\right)^{w_i}.$$
Since $|\xi_i| \le M$ and ${\mu} = 0$, we have
$$\expectwithbrackets{}{e^{c \xi_i}} = 1 + \sum_{p=2}^{+\infty}\frac{c^p\expectwithbrackets{}{\xi_i^{p}}}{p!}\le 1+ \sum_{p=2}^{+\infty}\frac{c^pM^{p-2}\sigma^2}{p!}$$
from the Taylor expansion for exponential functions.
Using $1+a\le e^a$, it follows that
$$\expectwithbrackets{}{e^{c \xi_i}}  \le \exp\left( \sum_{p=2}^{+\infty}\frac{c^pM^{p-2}\sigma^2}{p!} \right) = \exp\left(\frac{e^{cM}-1-cM}{M^2}\sigma^2\right)$$
and therefore
$$I \le \exp\left(-c\epsilon + \frac{e^{cM}-1-cM}{M^2}|w|\sigma^2 \right).$$

Now choose the constant $c$ to be the minimizer of the bound on the right hand side above:
$$ c= \frac{1}{M}\log\left(1+\frac{M\epsilon}{|w|\sigma^2}\right).$$
That is, $e^{cM}-1=\frac{M\epsilon}{|w|\sigma^2}$. With this choice,
$$I\le \exp \left( -\frac{|w|\sigma^2}{M^2} h\left(\frac{M\epsilon}{|w|\sigma^2}\right) \right).$$
This proves the desired inequality.
\end{proof}

% \begin{lemma}\label{lemma:three.sum}
% Let $\left(\xi_i\right)_{i=1}^{\SizeEdge}$ be $G$-networked random variables with mean $\expectwithbrackets{}{\xi_i} = {\mu}$ and variance $\sigma^2(\xi_i) = \sigma^2$,
% such that $|\xi_i-\mu|\le M$. 
% Let $w = (w_i)_{i=1}^{\SizeEdge}$ be a vertex-bounded weight vector for $G$ and let $|w| = \sum_i w_i$, then for all $\epsilon > 0$, 
% \begin{align*}
% &\Pr\left(\sum_{i=1}^{\SizeEdge} w_i\left(\xi_i - {\mu}\right) \ge \epsilon \right)
% \le \exp \left( -\frac{\epsilon}{2M}\log({1+\frac{M\epsilon}{|w|\sigma^2}}) \right), \\
% &\Pr\left(\sum_{i=1}^{\SizeEdge} w_i\left(\xi_i - {\mu}\right) \ge \epsilon \right)
% \le \exp \left( -\frac{\epsilon^2}{2(|w|\sigma^2+\frac{1}{3}M\epsilon)} \right),\\
% &\Pr\left(\sum_{i=1}^{\SizeEdge} w_i\left(\xi_i - {\mu}\right) \ge \epsilon \right)
% \le \exp \left( -\frac{\epsilon^2}{2|w|M^2} \right).
% \end{align*}
% \end{lemma}

We can also derive concentration inequalities which are networked analogues of the Bernstein and Chernoff-Hoeffding inequalities and the proofs are in the appendix. These inequalities are used in the next section to provide learning guarantees. 

\begin{theorem}\label{thm:three.avg}
Let $\left\{\xi_i\right\}_{i=1}^{\SizeEdge}$ be 
$G$-networked random variables with mean $\expectwithbrackets{}{\xi_i} = \mu$,
variance $\sigma^2(\xi_i) = \sigma^2$ (the variance condition is not needed for Inequality \eqref{eq:hoeffding.avg}), and satisfying $|\xi_i-\mu|\le M$. 
Let $w$ be a fractional matching of $G$ and $|w|=\sum_{i=1}^{\SizeEdge} w_i$, then for all $\epsilon > 0$, 
\begin{align}
& \Pr\left(\frac{1}{|w|}\sum_{i=1}^{\SizeEdge} w_i \xi_i - \mu \ge \epsilon \right)
\le \exp \left( -\frac{|w|\epsilon}{2M}\log\left({1+\frac{M\epsilon}{\sigma^2}} \right)\right), \label{eq:bennet.avg} \\
& \Pr\left(\frac{1}{|w|}\sum_{i=1}^{\SizeEdge} w_i \xi_i - \mu \ge \epsilon \right)
\le \exp \left( -\frac{|w|\epsilon^2}{2(\sigma^2+\frac{1}{3}M\epsilon)} \right)\hbox{, and} \label{eq:bernstein.avg}\\
& \Pr\left(\frac{1}{|w|}\sum_{i=1}^{\SizeEdge} w_i \xi_i - \mu \ge \epsilon \right)
\le \exp \left( -\frac{|w|\epsilon^2}{2M^2} \right).\label{eq:hoeffding.avg}
\end{align}
\end{theorem}

%Section 3.2.2: FMN
\subsection{The FMN Scheme}
According to Theorem \ref{thm:three.avg}, tighter bounds can be obtained by maximizing $|w|$ under the constraint that $w$ is a fractional matching. 
Given a hypergraph, this can be achieved by solving the linear program (LP): 
$$\max_{w} \quad \sum_{i=1}^{\SizeEdge} w_i \qquad
\text{s.t.} \quad \forall i: w_i\ge 0 \quad \text{and} \quad \forall v\in V: \sum_{i:v\in e_i} w_i\le 1.$$
%\begin{eqnarray*}
%&\max_{w} & \sum_{i=1}^{\SizeEdge} w_i\\
%&\text{s.t.} & \forall i: w_i\ge 0\\
%&& \forall v\in V: \sum_{i:v\in e_i} w_i\le 1
%\end{eqnarray*}
A weight vector which makes the sum $\sum_{i=1}^n w_i$ maximum is called a \emph{maximum fractional matching} of $G$. 
The optimal value of this linear program is called the \emph{fractional matching number (FMN)} of the hypergraph $G$ and is denoted by $\nu^*_G$. 
That is, $\nu^*_G$ is defined as the sum of a maximum fractional matching vector. 

There are very efficient LP solvers, including the simplex method which is efficient in practice, and the more recent interior-point methods \citep{boyd04}. The interior-point method solves an LP in $O(\mathsf{p}^2 \mathsf{q})$ time, where $\mathsf{p}$ is the number of decision variables, and $\mathsf{q}$ is the number of constraints. 
In practice, usually every hyperedge does not connect many vertices and a vertex is not incident to many hyperedges, so these LPs are usually sparse. 
Almost all LP solvers perform significantly better for sparse LPs.

%Section 3.2.3: Improvement on FMN
\subsection{Improvement on EQW}\label{sssec:improve}
Using Theorem \ref{thm:three.avg}, we can also improve Janson's inequalities for the EQW weighting scheme.
Let $w$ be a fractional matching and satisfy  $w_1=w_2=\ldots=w_n$ (EQW).
This requires that for all $i$, $0< w_i \le 1/\omega_G$ 
where $\omega_G= \max_{v\in V_G} |\{e: v\in e \}|$ is the maximum degree of $G$.
Let $w_1=w_2=\ldots=w_n=1/\omega_G$, we can get the following corollary. 

\begin{corollary}\label{cor:sameweight}
Let $\left\{\xi_i\right\}_{i=1}^{\SizeEdge}$ be $G$-networked random variables with mean $\expectwithbrackets{}{\xi_i} = \mu$,
variance $\sigma^2(\xi_i) = \sigma^2$, and satisfying $|\xi_i-\mu|\le M$. Then for all $\epsilon > 0$,
\begin{align*}
& \Pr\left(\frac{1}{\SizeEdge}\sum_{i=1}^{\SizeEdge} \xi_i - \mu \ge \epsilon \right)
\le \exp \left( -\frac{n\epsilon^2}{2\omega_GM}\log\left({1+\frac{M\epsilon}{\sigma^2}} \right)\right), \\
& \Pr\left(\frac{1}{\SizeEdge}\sum_{i=1}^{\SizeEdge} \xi_i - \mu \ge \epsilon \right)
\le \exp \left( -\frac{n\epsilon^2}{2\omega_G(\sigma^2+\frac{1}{3}M\epsilon)} \right), \\
& \Pr\left(\frac{1}{\SizeEdge}\sum_{i=1}^{\SizeEdge} \xi_i - \mu \ge \epsilon \right)
\le \exp \left( -\frac{n\epsilon^2}{2\omega_GM^2} \right).
\end{align*}
\end{corollary}

It is known that for every hypergraph, the maximum degree\footnote{Do not confuse the maximum degree $\omega_G$ with another concept - the maximum degree $\Delta$ of the corresponding dependency graphs (see \citep{janson04}). It holds that $\chi^*_G\le \Delta + 1$, but for any $c>0$ there exist (infinitely many) hypergraphs such that $\chi^*_G > \omega_G + c$.} is not larger than the fractional hyperedge-chromatic number, i.e., $\omega_G\le \chi^*_G$. 
This fact generally ensures that the inequalities in Corollary \ref{cor:sameweight} provide tighter bounds than those in Theorem \ref{thm:janson04}. 
In addition, for any number $r\ge 1$, there exist hypergraphs $G$ such that $\chi^*_G/\omega_G > r$, and hence the improvement of Corollary \ref{cor:sameweight} over Theorem \ref{thm:janson04} can be arbitrarily large. 
As an example, we consider (truncated) projective planes\footnote{Projective planes and truncated projective planes are not only of theoretical interest. 
In fact, they are special cases of block designs as studied in the field
of experimental design \citep{colbourn10}. 
This field studies what points in a feature space to measure to maximize certain experimental objectives such as diversity and independence of training data.} 
of order $\kappa$ (see e.g., \citep{Matousek98}). 
It is known that there exists a projective plane of order $\kappa$ whenever $\kappa$ is a prime power. 
The fractional hyperedge-chromatic number of any subhypergraph of a projective plane is equal to its hyperedge number.
For such datasets, Janson's inequalities (Theorem \ref{thm:janson04}) fail to offer useful bounds while
Corollary \ref{cor:sameweight} provides 
significantly better bounds.  
Indeed, the maximum degree of the projective plane of order $\kappa$ ($\kappa\ge 2$) is $\kappa+1$, so $\chi^*_G / \omega_G = \kappa + 1/(\kappa+1)=\Omega(|E_G|^{1/2})$. 

For different hypergraphs, it is possible that $\SizeEdge / \omega_G > \nu_G$, $\SizeEdge / \omega_G = \nu_G$ or $\SizeEdge / \omega_G < \nu_G$. 
Therefore, whether the IND scheme (i.e., the classical concentration bounds applied to independent subsets) provides tighter bounds than the EQW scheme depends on the hypergraph. 
However, both $\SizeEdge / \omega_G$ and $\nu_G$ are smaller than $\nu^*_G$.  
Hence, the FMN scheme always gives better concentration bounds than the other two schemes. 

\subsection{U-statistics}
\cite{ustatistics} presented concentration inequalities for U-statistics which is an important class of statistics on networked random variables. 
Above results also improve these concentration inequalities.
As an example, we consider one-sample U-statistics. 

\begin{definition}[One-sample U-statistics]
Let $\{S_i\}_{i=1}^m$ be independent random variables. 
For $r\le m$ consider a random variable of the form
$$U = \frac{(m-r)!r!}{m!} \sum_{m,r} \xi(S_{i_1}, \ldots, S_{i_r})$$
where the sum $\sum_{m,r}$ is taken over all subset $\{i_1,\ldots,i_r\}$ of distinct positive integers not exceeding $m$.
The random variable $U$ is called a one-sample U-statistic.
\end{definition}
If the function $\xi$ is bounded, $|\xi - \expectwithbrackets{}{\xi}|\le M$,
\cite{ustatistics} showed that for any $\epsilon > 0$, 
$$\Pr\left( U - \mu \ge \epsilon \right) \le \exp\left(-\frac{ \lfloor \frac{m}{r}\rfloor \epsilon^2}{2M^2}\right)$$
where $\mu=\expectwithbrackets{}{U}$.
A corollary of our result shows that the operator $\lfloor \rfloor$ is not necessary
$$\Pr\left( U - \mu \ge \epsilon \right) \le \exp\left(-\frac{ m \epsilon^2}{2rM^2}\right).$$
To prove this inequality, we construct a hypergraph $G=(V_G, E_G)$ for a one-sample U-statistic. 
This hypergraph has $m$ vertices, and $E_G=\{e\subseteq V_G\mid |e|=r\}.$ 
We consider the independent random variables $\{S_i\}_{i=1}^m$ as the features of the vertices. The statistic $U$ is an equally weighted sample mean of the networked random variables defined on the hyperedges. 
The inequality is proved by letting $n = \frac{m!}{(m-r)!r!}$ and $\omega_G = \frac{(m-1)!}{(m-r)!(r-1)!}$ in Corollary \ref{cor:sameweight}. 

%\subsection{Tightness Result}

%section of learning theory
\section{Application of new inequalities: Learning theory for networked examples}\label{sec:learningtheory}
In this section, we use the results we obtained in the previous section to show generalization performance guarantees
when learning from networked examples, making the same relaxed assumptions as in previous sections.
We do this in the context of a specific framework (empirical risk minimization), 
but the same principles can be applied to many other paradigms of learning theory, e.g., structural risk minimization \citep{shawe98}. 
We showed that the FMN weighting scheme provides clearly better properties than classical approaches. 

%Section 2.3: Learning Task
% \subsection{Learning task}\label{sec: basic concepts of learning theory}
The main goal of supervised learning is to learn a
function $f:\spacefull\mapsto \targetvalset$
from a set of training examples $Z=\{z_i\}_{i=1}^{\SizeEdge}$ with
$z_i=(x_i,y_i) \sim \probl$, and
to predict labels for unseen examples. 

For networked examples in hypergraph $G$, $x_i\in \spacefull$ is the combination of the features of vertices in hyperedge $e_i$ and $y_i\in \targetvalset$ is the target value of hyperedge $e_i$. 
In this case, the dataset $Z$ is called a $G$-networked sample. 

We consider the least square regression which goal is to find a minimizer of  the \emph{expected risk}\footnote{Similar results for general loss functions can be established, in this paper, we take the squared loss for the sake of simplicity.}
$$\mathcal{E}(f)=\int_{\cal Z}\left(f(x)-y\right)^2 \probl(x,y) dxdy.$$
That is, \begin{equation*}
f_{\rho,{\cal F}}=\arg\min_{f\in {\cal F}} \mathcal{E}(f),
\end{equation*}
where the minimization is taken over the set of all measurable functions $\cal F$.
Unfortunately, since the probability distribution $\probl$ is unknown,
$f_{\rho,{\cal F}}$ cannot be computed directly. We find a good approximation of $f_{\rho,{\cal F}}$ from sample. 
The empirical risk minimization (ERM) principle \citep{vapnik91} is to find
the minimizer of empirical risk $\mathcal{E}_{Z}(f)$ in a properly selected hypothesis space $\mathcal{H}$, i.e.,
\begin{equation*}
f_{Z,\cal{H}}=\arg\min_{f\in\mathcal{H}}
\left\{\mathcal{E}_{Z}(f)=\frac{1}{\SizeEdge}\sum_{i=1}^{\SizeEdge}\left(f(x_i)-y_i\right)^2\right\}.
\end{equation*}
The performance of the ERM approach is commonly measured by the \emph{excess risk}
$$\mathcal{E}(f_{Z,\cal{H}})-\mathcal{E}(f_{\rho,{\cal F}})=
[\mathcal{E}(f_{Z,\cal{H}})-\mathcal{E}(f_{\rho,{\cal H}})]+
[\mathcal{E}(f_{\rho,{\cal H}})-\mathcal{E}(f_{\rho,{\cal F}})],$$
where
\begin{equation*}
f_{\rho,{\cal H}}=\arg\min_{f\in\mathcal{H}} \mathcal{E}(f).
\end{equation*}
We call the first part
the \emph{sample error} $\mathcal{E}_{S}(Z) :=\mathcal{E}(f_{Z,\cal{H}})-\mathcal{E}(f_{\rho,{\cal H}}),$
the second part the \emph{approximation error} $\mathcal{E}_{A}({\cal H}):=\mathcal{E}(f_{\rho,{\cal H}})-\mathcal{E}(f_{\rho,{\cal F}})$.

The approximation error is independent of the sample. In this paper, we concentrate on the sample error (we refer the readers who are interested in analyzing approximation errors to \citep{cucker07}). To this end, we need to choose a proper hypothesis space. 
The complexity of the hypothesis space is usually measured in terms of covering number \citep{zhou02}, entropy number \citep{tsuda99}, VC-dimension \citep{vapnik94}, etc.
We use the covering numbers defined below to measure the capacity of our hypothesis space ${\mathcal{H}}$, 
and the hypothesis space $\mathcal{H}$ is usually chosen as a subset of $\mathcal{C}(\spacefull)$
which is a Banach space of continuous functions on a compact metric space $\spacefull$ with the norm
$\|f\|_{\infty}=\sup_{x\in {\spacefull}}|f(x)|.$ 
Our results can be extended to other complexity measures such as VC-dimensions, but it has been shown in \cite{evgeniou1999v} that VC-dimension is not suitable for real-valued function classes, while in this section the random variables are presented as real-valued functions, therefore, we measure the complexity of the hypothesis space by the uniform covering number. 

\begin{definition}[Covering number]\label{def:coveringnumber}
Let $\cal{H}$ be a metric space and $\tau>0.$
We define the \emph{covering number} ${N}(\cal{H},\tau)$ to be the minimal $\ell\in\mathbb{N}$ such that there exists $\ell$
disks in $\cal{H}$ with radius $\tau$ covering $\cal{H}$.
When $\cal{H}$ is compact, this number is finite.
\end{definition}

\begin{definition}[M-bounded functions]\label{def:Mbounded}
Let $M>0$ and $\probl$ be a probability distribution on ${\cal Z}.$
We say that a set $\cal{H}$ of functions from $\spacefull$ to $\targetvalset$ is \emph{M-bounded} when
$$\Pr_{(x,y)\sim\probl}\left(\sup_{f\in\mathcal{H}}|f({x})-y|\le M \right) = 1.$$
\end{definition}

\begin{theorem}[\cite{cucker07}]\label{thm:iid}
Let $\mathcal{H}$ be a compact and convex subset of $\mathcal{C}(\spacefull)$, and $Z$ be an i.i.d.\ sample.
If $\mathcal{H}$ is M-bounded, then for all $\epsilon>0$,
$$\Pr\big(\mathcal{E}_{S}(Z)\geq \epsilon\big)\le
{N}\left(\mathcal{H},\frac{\epsilon}{12M}\right)  \exp\left(-\frac{n\epsilon^2}{300M^4}\right).$$
\end{theorem}
The above theorem provides learning guarantee for i.i.d.\ training samples. 
% Our task is to propose methods for learning from networked examples and show learning guarantees. 
% We (informally) define properties of learning methods: 
% \begin{itemize}
% \item Normalized: if the networked examples happen to be mutually independent, i.e., any two hyperedges are disjoint, then the learning bound is the same the one of learning from independent examples. If a learning method is normalized, then it naturally generalizes the corresponding learning method with i.i.d.\ samples. 

% \item Efficient: The learning method can be implemented efficiently. In this paper, we call a method efficient if we do not need to solve any NP-hard problem. 
% \end{itemize}

Now, we consider three weighting schemes having different sample error bounds which are related to different important parameters of hypergraphs. 
The first two weighting schemes are straightforward, but from the upper bound point of view, they waste the information provided by the networked examples. 
The third weighting scheme reaches a better sample error bound via solving the linear program discussed in Section \ref{ssec:networked.concentration}. 

%Section 5.2.1: EQW (slightly different from the outline)
\subsection{The EQW Scheme}\label{sssec:eqw}
Let us first consider the EQW weighting scheme that learns from a set of networked examples in the same way as if they were i.i.d., i.e., without weighting them as a function of the network structure.
We can use Corollary \ref{cor:sameweight} above to bound the sample error of EQW scheme:
\begin{theorem}\label{thm:chromatics}
Let $\mathcal{H}$ be a compact and convex subset of $\mathcal{C}(\spacefull)$, and $Z$ be a $G$-networked sample.
If $\mathcal{H}$ is M-bounded, then for all $\epsilon>0$,
$$\Pr\big(\mathcal{E}_{S}(Z)\ge \epsilon\big)\le  {N}\left(\mathcal{H},\frac{\epsilon}{12M}\right)
\exp\left(-\frac{n\epsilon^2}{300\omega_GM^4}\right).$$
\end{theorem}
The result above shows that the bound of the sample error not only relies on the sample size but also depends on the maximum degree $\omega_G$.
That is, a larger sample may result in a poorer sample error bound since  $\omega_G$
can also become larger.
% \begin{proposition}
% The EQW scheme is normalized and efficient, but not monotonic. 
% \end{proposition}
% To see that the EQW scheme is not monotonic, we consider a hypergraph $G$ with $n$ independent hyperedges and a super-hypergraph $G'$ also with $n>1$  independent hyperedges and an extra hyperedge which intersects with any other hyperedge. 
% Using the EQW scheme, learning from $G$ has a better bound than learning from $G'$ since 
% $|E_G|/\chi^*_G = n  > (n+1)/2 = |E_{G'}|/\chi^*_{G'}$. 
Remember that almost all previous works only deal with the EQW scheme, and their results depend on $n/\chi^*_G$, so Theorem \ref{thm:chromatics}, which depends on $n/\omega_G$, improves their results significantly. 

%Section 5.2.2: IND
\subsection{The IND scheme}\label{ssec:ind}
A straightforward idea to learn from a $G$-networked sample $Z$ is to
find a maximum subset $Z_I \subseteq Z$ of training examples that
correspond to a matching in $G$.
Due to our assumptions, such set is an i.i.d.\ sample.
We can then perform algorithms on $Z_I$ for learning, the function we obtain by the ERM principle is
$$f_{Z_I,{\cal H}} = \arg \min_{f \in {\cal H}} \left\{{\cal E}_{Z_I}(f)=\frac{1}{|Z_I|} \sum_{z_i \in Z_I}\left(f(x_i)-y_i \right)^2\right\}.$$
To bound the sample error of this weighting scheme, we can directly use Theorem \ref{thm:iid}, replacing $n$ there by $|Z_I|$. 
As we discussed in Section \ref{sec:concentration.inequalities}, despite the IND scheme satisfies monotonicity, it is still far from optimal, due to the lack of efficiency. 
Besides, similar to the EQW scheme, the IND scheme may result in that networked datasets are not fully utilized. 

% \begin{proposition}
% The IND scheme is normalized and monotonic, but not efficient. 
% \end{proposition}

%Section 5.2.3: FMN
\subsection{The FMN scheme}\label{sec:s}
We now consider the FMN weighting scheme proposed in Section \ref{ssec:networked.concentration}.
For a $G$-networked sample $Z$, we denote the FMN weighted sample $Z_{\nu^*} = \{(z_i,w_i)\}$
where $(w_i)_{i=1}^{\SizeEdge}$ is an FMN weight vector.
Now we can define a new empirical risk on the FMN weighted sample $Z_{\nu^*}$ that
$$\mathcal{E}_{Z_{\nu^*}}(f)=\frac{1}{\nu^*}\sum_{i=1}^{\SizeEdge} w_i\left(f({x}_i)-y_i\right)^2.$$
Later, we show that the empirical risk $\mathcal{E}_{Z_{\nu^*}}$ converges to the expected risk
$\mathcal{E}(f)$ as $\nu^*$ tends to infinity for fixed $f,$  then excess risk can be divided into two parts as follows
$$\mathcal{E}(f_{{Z}_{\nu^*},{\cal H}})-\mathcal{E}(f_{\rho,{\cal F}})=
[\mathcal{E}(f_{{Z}_{\nu^*},{\cal H}})-\mathcal{E}(f_{\rho,\mathcal{H}})]+[\mathcal{E}(f_{\rho,\mathcal{H}})-\mathcal{E}(f_{\rho,{\cal F}})].$$
We use the probability inequalities with $\nu^*$ (see Theorem \ref{thm:three.avg})
to estimate the sample error $\mathcal{E}_{S}({Z}_{\nu^*}):=\mathcal{E}(f_{{Z}_{\nu^*},{\cal H}})-\mathcal{E}(f_{\rho,{\cal H}}).$

\begin{theorem}\label{thm:nu-bound}
Let $\mathcal{H}$ be a compact and convex subset of $\mathcal{C}(\spacefull).$
If $\mathcal{H}$ is M-bounded,
then for all $\epsilon>0$,
$$\Pr\big(\mathcal{E}_{S}({Z}_{\nu^*})\ge \epsilon\big)\le {N}\left(\mathcal{H},\frac{\epsilon}{12M}\right)  \exp\left(-\frac{\nu^*\epsilon^2}{300M^4}\right).$$
\end{theorem}

% \begin{proposition}
% The FMN scheme is normalized, monotonic and efficient. 
% \end{proposition}

%Section 7: Conclusions
\section{Concluding remarks}\label{sec:conclusion}
In this paper, we considered the problem of learning from networked data. 
We proposed several schemes for weighting training examples that allow for using the available training data to a large extent while mitigating the dependency problem. 
In particular, the FMN weighting scheme allows for generalizing a large fraction of existing statistical learning theory. 
The weights in our weighting schemes can be computed efficiently.
The presented theory forms a step towards a statistically sound theory for learning in networks. 

We made here a weaker independence assumption. In order to analyze more real-world networked cases, one needs to build new models with more flexible assumptions than the classic i.i.d.\ assumption and the weaker assumption we studied which is already more general but still is not sufficiently powerful to properly model a range of real-world scenarios. 
A first step in this direction would be to develop a measure to assess the strength of the dependency of the features between vertices and its influence on the learning task at hand. 

To design active learning methods for structured data is another useful topic,
i.e., to study query strategies to choose objects or examples in a network to
perform experiments in order to learn a good predictor at minimal cost. 
It is also an interesting work to study the implications of the above
theory and algorithms for other learning settings and tasks, e.g., cross-validation and bootstrapping (see e.g., \citep{liu1988bootstrap}). 

% \acks{This work is supported by ERC Starting Grant 240186
% “MiGraNT: Mining Graphs and Networks: a Theory-based approach”.
% The authors thank Prof. Ding-Xuan Zhou for the fruitful discussions and thank Prof. Maurice Bruynooghe for comments.} 

\newpage
\bibliography{jmlr}
\appendix

\newpage

%Section 3.1: Equally weighted averages and Janson's bound
\section{EQW and Janson's Inequalities}\label{ssec:eqw}
\cite{janson04} showed inequalities on the gap between the expected value $\mu$ and the unweighted average of $n$ networked random variables $\xi_i$. 

\begin{theorem}[\cite{janson04}]\label{thm:janson04}
Let $\left\{\xi_i\right\}_{i=1}^{\SizeEdge}$ be $G$-networked random variables with mean $\expectwithbrackets{}{\xi_i} = \mu$, variance $\hbox{var}(\xi_i)=\sigma^2$
and satisfying $|\xi_i-\mu|\le M$.  Then for all $\epsilon > 0$,
\begin{align*}
& \Pr\left(\frac{1}{\SizeEdge}\sum_{i=1}^{\SizeEdge} \xi_i - \mu \ge \epsilon \right)
\le \exp \left(-\frac{n\epsilon^2}{2\chi^{*}_G M^2} \right),\\
& \Pr\left(\frac{1}{\SizeEdge}\sum_{i=1}^{\SizeEdge} \xi_i - \mu \ge \epsilon \right)
\le \exp \left(-\frac{8n\epsilon^2}{25\chi^{*}_G(\sigma^2+M\epsilon/3)} \right). 
\end{align*}
\end{theorem}

\noindent\textbf{Remark} Note that the second inequality above is a Bernstein-type inequality, but it cannot degenerate to the original Bernstein inequality when the random variables are independent. 
% Therefore, if one applies this Bernstein-type inequality to derive generalization error bounds, then it cannot tell us whether a learning method is normalized or not. 
In Section \ref{sssec:improve}, we show an improvement on these inequalities, which not only removes the unnecessary coefficient in the second inequality and then makes the inequality can degenerate to the original independent case when random variables are independent, but also replaces $\chi^{*}_G$ by another parameter which is much smaller than $\chi^{*}_G$ in many cases.

\section{Proofs Omitted from Section \ref{sec:concentration.inequalities}}\label{app:proof_expineq} 
In this part, we prove Theorem \ref{thm:three.avg}. 
\begin{lemma}\label{lemma:concave}
Let $\mathbf{\beta} = (\beta_i)_{i=1}^k \in \mathbb{R}^k_+$ such that $\sum_{i=1}^k \beta_i \le 1$.
Then, the function $g(t)$ with
$t = (t_i)_{i=1}^k \in \mathbb{R}^k_+ $ defined by
$g(t) = \prod_{i=1}^k t_i^{\beta_i}$, is concave.
\end{lemma}
%\noindent\textbf{Lemma \ref{lemma:concave}} \textit{Let $\mathbf{\beta} = %(\beta_i)_{i=1}^k \in \mathbb{R}^k_+$ such that $\sum_{i=1}^k \beta_i \le 1$.
%Then, the function $g(t)$ with
%$t = (t_i)_{i=1}^k \in \mathbb{R}^k_+ $ defined by
%$g(t) = \prod_{i=1}^k t_i^{\beta_i}$ is concave.}
\begin{proof}
We prove by showing that its Hessian matrix $\nabla^2 g(t)$ is negative semidefinite.
$\nabla^2 g(t)$ is given by
$$\frac{\partial^2 g(t)}{\partial t_i^2} = \frac{\beta_i(\beta_i-1)g(t)}{t_i^2},
\qquad \frac{\partial^2 g(t)}{\partial t_i \partial t_j} = \frac{\beta_i \beta_j g(t)}{t_i t_j} ,$$
and can be expressed as
$$\nabla^2 g(t) =\left(qq^\mathrm{T} - \textbf{diag}(\beta_1/t_1^2, \ldots, \beta_n/t_n^2)\right)g(t)$$
where $q=[q_1,\ldots,q_k]$ and $q_i = \beta_i/t_i$.
We must show that $\nabla^2 g(t) \preceq 0$, i.e., that
$$u^\mathrm{ T } \nabla^2 g(t) u =
\left( \left( \sum_{i=1}^k  \beta_i u_i/t_i \right)^2 - \sum_{i=1}^k \beta_i u_i^2/t_i^2 \right) g(t) \le 0$$
for all $u \in \mathbb{R}^k$.
Because $g(t) \ge 0$ for all $t$, we only need to prove
$$ \left( \sum_{i=1}^k  \beta_i u_i/t_i \right)^2 - \sum_{i=1}^k \beta_i u_i^2/t_i^2  \le 0.$$
Since $\beta_i$ is positive for every $i$ and $\sum_{i=1}^k \beta_i \le 1$,
we define a random variable $\xi$ with probability $P(\xi = u_i/t_i) = \beta_i$ and $P(\xi = 0) = 1-\sum_{i=1}^k \beta_i$.
From basic probability theory, we have
$$\left( \sum_{i=1}^k  \beta_i u_i/t_i \right)^2 = \left(\expectwithbrackets{}{\xi}\right)^2 \le \expectwithbrackets{}{\xi^2 } = \sum_{i=1}^k \beta_i u_i^2/t_i^2.$$
\end{proof}

\begin{lemma}\label{lemma:three.sum}
Let $\left(\xi_i\right)_{i=1}^{\SizeEdge}$ be $G$-networked random variables with mean $\expectwithbrackets{}{\xi_i} = {\mu}$ and variance $\sigma^2(\xi_i) = \sigma^2$,
such that $|\xi_i-\mu|\le M$. 
Let $w = (w_i)_{i=1}^{\SizeEdge}$ be a vertex-bounded weight vector for $G$ and let $|w| = \sum_i w_i$, then for all $\epsilon > 0$, 
\begin{align*}
&\Pr\left(\sum_{i=1}^{\SizeEdge} w_i\left(\xi_i - {\mu}\right) \ge \epsilon \right)
\le \exp \left( -\frac{\epsilon}{2M}\log({1+\frac{M\epsilon}{|w|\sigma^2}}) \right), \\
&\Pr\left(\sum_{i=1}^{\SizeEdge} w_i\left(\xi_i - {\mu}\right) \ge \epsilon \right)
\le \exp \left( -\frac{\epsilon^2}{2(|w|\sigma^2+\frac{1}{3}M\epsilon)} \right),\\
&\Pr\left(\sum_{i=1}^{\SizeEdge} w_i\left(\xi_i - {\mu}\right) \ge \epsilon \right)
\le \exp \left( -\frac{\epsilon^2}{2|w|M^2} \right).
\end{align*}
\end{lemma}
\begin{proof}
Without loss of generality, we assume ${\mu} = 0$.
The first inequality follows from Lemma \ref{lemma:bennett} and the inequality
$$h(a) \ge \frac{a}{2}\log(1+a),\;\forall a\ge0.$$
 
The second inequality follows from Lemma \ref{lemma:bennett} and the inequality
$$h(a)\ge\frac{3a^2}{6+2a},\;\forall a\ge0.$$

To prove the third inequality, we use Lemma \ref{lemma:exponential}.
As the exponential function is convex and $-M \le \xi_i\le M$ %almost surely
, there holds
$$e^{c\xi_i}\le \frac{c\xi_i-(-cM)}{2cM}e^{cM} + \frac{cM-c\xi_i}{2cM}e^{-cM}.$$
It follows from the assumption $\mu=0$ and the Taylor expansion for the exponential function that
\begin{align*}
\expectwithbrackets{}{e^{c\xi_i}}  \le &  \frac{1}{2} e^{-cM} + \frac{1}{2}  e^{cM} 
= \frac{1}{2}\sum_{p=0}^{+\infty} \frac{(-cM)^p}{p!} + \frac{1}{2}\sum_{p=0}^{+\infty} \frac{(cM)^p}{p!}
= \sum_{p=0}^{+\infty} \frac{(cM)^{2p}}{(2p)!}\\
= & \sum_{p=0}^{+\infty}\frac{((cM)^2/2)^p}{p!}\prod_{j=1}^{p} \frac{1}{2j-1} \le \sum_{p=0}^{+\infty}\frac{((cM)^2/2)^p}{p!} = \exp((cM)^2/2).
\end{align*}
This, together with Lemma \ref{lemma:exponential}, implies 
\begin{eqnarray*} \Pr\left(\sum_{i=1}^{\SizeEdge} w_i\left(\xi_i - {\mu}\right) \ge \epsilon \right)  &=& \Pr\left(\exp\left(c\sum_{i=1}^{\SizeEdge} w_i \xi_i\right) \ge e^{c\epsilon} \right)\\
&\le& \exp\left(-c\epsilon + \expectwithbrackets{}{c\sum_{i=1}^{\SizeEdge} w_i \xi_i}\right)\\
&\le&  \exp\left(-c\epsilon + |w|(cM)^2/2\right).
\end{eqnarray*}
Choose $c=\epsilon/(|w|M^2)$.
Then, $\Pr\left(\sum_{i=1}^{\SizeEdge} w_i\left(\xi_i - {\mu}\right) \ge \epsilon \right)\le \exp\left(-\frac{\epsilon^2}{2|w|M^2}\right)$.
\end{proof}

Now we are ready to prove Theorem \ref{thm:three.avg}.
 
\begin{proof}[proof of Theorem \ref{thm:three.avg}]
We apply 
Lemma \ref{lemma:three.sum} to the variables
$\xi_i'=\xi_i/|w|$ which satisfy $|\xi_i'-\expectwithbrackets{}{\xi_i'}| \le M/{|w|}, \sigma^2(\xi_i')=\sigma^2/{|w|}^2$. 
\end{proof}

\section{Proofs Omitted from Section \ref{sec:learningtheory}}\label{app:proof_sampleerror}
In this part we prove Theorem \ref{thm:nu-bound}.
We first give some lemmas which are extended versions of lemmas that were used before to establish the sample error bounds for i.i.d.\ samples.
In particular, some ideas were borrowed from \citep{cucker07}.
For any function $f\in{\cal H}$, we define the defect function $\mathcal{D}_{{Z}_{\nu^*}}(f)={\cal E}(f)-{\cal E}_{{Z}_{\nu^*}}(f)$, 
the difference between the expected risk of $f$ and the empirical risk of $f$ on the FMN weighted sample $Z_{\nu^*}$.

\begin{lemma}\label{lemma:single}
Let $M > 0$ and let $f : \spacefull \mapsto \targetvalset$ be $M$-bounded. Then for all $\epsilon > 0$,
$$ \Pr\left( \mathcal{D}_{{Z}_{\nu^*}}(f) \ge -\epsilon \right) \ge 1-\exp\left( \frac{\nu^* \epsilon^2}{2M^4} \right).$$
\end{lemma}
\begin{proof}
Note that $\Pr\left( \mathcal{D}_{{Z}_{\nu^*}}(f) \ge -\epsilon \right) = \Pr\left({\cal E}_{{Z}_{\nu^*}}(f) - {\cal E}(f) \le \epsilon \right).$
This lemma then follows directly from Inequality\ \eqref{eq:hoeffding.avg} in Theorem \ref{thm:three.avg} by taking $\xi_i=(f(x_i)-y_i)^2$
satisfying $|\xi_i|\le M^2$ when $f$ is M-bounded.
\end{proof}

To present Lemma \ref{lemma:uniform} and \ref{lemma:familyclass1}, we first define full measure sets.

\begin{definition}[full measure set]
A set $U\subseteq {\cal Z}$ is full measure for distribution $\probl$ over ${\cal Z}$ if $\Pr_{z\sim\probl}\left(z \in U\right)=1$.
\end{definition}

\begin{lemma}\label{lemma:uniform}
If for $j=1,2$, $|f_j(x)-y|\le M$ on a full measure set $U\subseteq {\cal Z}$ then, for all $Z\in U^{\SizeEdge}$
$$|\mathcal{D}_{{Z}_{\nu^*}}(f_1)-\mathcal{D}_{{Z}_{\nu^*}}(f_2)|\le 4M\|f_1-f_2\|_\infty.$$
\end{lemma}
\begin{proof}
Because
$$(f_1(x) - y)^2 - (f_2(x) - y)^2 = (f_1(x) + f_2(x) - 2y)(f_1(x) - f_2(x)),$$
we have
\begin{eqnarray*}
|{\cal E}(f_1)-{\cal E}(f_2)| & = & \left|\int_{\cal Z}\probl(z)(f_1(x) + f_2(x) - 2y)(f_1(x) - f_2(x))\hbox{d}z\right|\\
&\le&  \int_{\cal Z}\probl(z)|(f_1(x) -y) + (f_2(x) - y)|\|f_1-f_2\|_\infty\hbox{d}z \\
&\le& 2M\|f_1-f_2\|_\infty.
\end{eqnarray*}
For $Z\in U^{\SizeEdge}$, we have 
\begin{eqnarray*}
|{\cal E}_{{Z}_{\nu^*}}(f_1)-{\cal E}_{{Z}_{\nu^*}}(f_2)| & = & \frac{1}{\nu^*}\sum_{i=1}^{\SizeEdge} w_i(f_1(x_i) + f_2(x_i) - 2y_i)(f_1(x_i) - f_2(x_i)\\
&\le&  \frac{1}{\nu^*}\sum_{i=1}^{\SizeEdge} w_i|(f_1(x_i)-y_i) + (f_2(x_i) - y_i)|\|(f_1 - f_2\|_\infty \\
&\le& 2M\|f_1-f_2\|_\infty.
\end{eqnarray*}
Thus, $$|\mathcal{D}_{{Z}_{\nu^*}}(f_1)-\mathcal{D}_{{Z}_{\nu^*}}(f_2)| = |{\cal E}(f_1)-{\cal E}_{{Z}_{\nu^*}}(f_1)-{\cal E}(f_2)+{\cal E}_{{Z}_{\nu^*}}(f_2)|
\le 4M\|f_1-f_2\|_\infty.$$ 
\end{proof}

\begin{lemma}\label{lemma:familyclass1}
Let $\mathcal{H}$ be a compact M-bounded subset of $\mathcal{C}(\spacefull).$ Then, for all $\epsilon>0,$
$${\Pr}\left(\sup_{f\in\mathcal{H}}\mathcal{D}_{{Z}_{\nu^*}}(f)\leq \epsilon\right)\ge 1-{N}\left(\mathcal{H},\frac{\epsilon}{8M}\right) \exp\left(-\frac{\nu^*\epsilon^2}{8M^4}\right).$$
\end{lemma}
\begin{proof}
 Let $\{f_j\}_{j=1}^\ell\subset \mathcal{H}$ with $\ell={N}\left(\mathcal{H},\frac{\epsilon}{4M}\right)$ such that $\mathcal{H}$
is covered by disks $D_j$ centered at $f_j$ with radius $\frac{\epsilon}{4M}.$ 
Let $U$ be a full measure set on which $\sup_{f\in\mathcal{H}}|f(x)-y|\le M$. 
Then for all $Z\in U^{\SizeEdge}$ and for all $f\in D_j$, according to Lemma \ref{lemma:uniform}, we have
$$|\mathcal{D}_{{Z}_{\nu^*}}(f)-\mathcal{D}_{{Z}_{\nu^*}}(f_j)|\leq 4M\|f-f_j\|_\infty\leq 4M \frac{\epsilon}{4M}=\epsilon.$$
Consequently,
$$\sup_{f\in D_j} \mathcal{D}_{{Z}_{\nu^*}}(f)\ge 2\epsilon \Rightarrow \mathcal{D}_{{Z}_{\nu^*}}(f_j)\ge \epsilon.$$
Then we conclude that, for $j=1,\cdots,\ell,$
\[\Pr\left(\sup_{f\in {D_j}} \mathcal{D}_{{Z}_{\nu^*}}(f)\ge 2\epsilon\right)\leq \Pr\left( \mathcal{D}_{{Z}_{\nu^*}}(f_j)\ge \epsilon\right)
\le \exp\left(-\frac{\nu^*\epsilon^2}{2M^4}\right).\]
The last inequality follows from Inequality\ \eqref{eq:hoeffding.avg} in Theorem \ref{thm:three.avg} by taking $\xi_i=-(f_j(x_i)-y_i)^2$. 
In addition, one can easily see that
$$\sup_{f\in\mathcal{H}} \mathcal{D}_{{Z}_{\nu^*}}(f)\ge \epsilon \Leftrightarrow \exists j\le \ell : \sup_{f\in D_j} \mathcal{D}_{{Z}_{\nu^*}}(f)\ge \epsilon$$
and from the fact that the probability of a union of events is bounded by the sum of the probabilities of these events it follows that
$$\Pr\left(\sup_{f\in\mathcal{H}}\mathcal{D}_{{Z}_{\nu^*}}(f)\ge \epsilon\right)\leq \sum_{j=1}^\ell {\Pr}\left(\sup_{f\in D_j}\mathcal{D}_{{Z}_{\nu^*}}(f)\ge \epsilon\right) \leq \ell \exp\left(-\frac{\nu^*\epsilon^2}{8M^4}\right).$$
This completes the proof.
\end{proof}

\begin{lemma}\label{lemma:ratiosingle}
Suppose networked random variables $\left(\xi_i\right)_{i=1}^{\SizeEdge}$
satisfy that for all $i$, $\expectwithbrackets{}{\xi_i}=\mu\ge 0,$ and $|\xi_i-\mu|\leq B$ almost everywhere. 
Let $\left(w_i\right)_{i=1}^{\SizeEdge}$ be any FMN weight vector.
If $\expectwithbrackets{}{\xi_i^2}\leq c\mu,$ then for every $\epsilon>0$ and $0<\alpha\leq 1,$ there holds
\[\Pr \left(\frac{\mu-\frac{1}{\nu^*}\sum_{i=1}^{\SizeEdge} w_i \xi_i}{\sqrt{\mu+\epsilon}}>\alpha \sqrt{\epsilon}\right)\leq 
\exp\left(-\frac{\alpha^2 \nu^* \epsilon}{2c+\frac{2}{3}B}\right).\]
\end{lemma}
\begin{proof}
We apply Inequality\ \eqref{eq:bernstein.avg} in Theorem \ref{thm:three.avg} by substituting the $\xi_i$ in Inequality\ \eqref{eq:bernstein.avg} with $\xi_i/\sqrt{\mu+\epsilon}$, 
the $\epsilon$ in Inequality\ \eqref{eq:bernstein.avg} with $\alpha\sqrt{\epsilon}$, the $M$ in Inequality\ \eqref{eq:bernstein.avg} with $B/\sqrt{\mu+\epsilon}$ and the $|w|$ in Inequality\ \eqref{eq:bernstein.avg} with $\nu^*$.
We get 
\[
\Pr\left(\frac{\mu-\frac{1}{\nu^*}\sum_{i=1}^{\SizeEdge} w_i \xi_i}{\sqrt{\mu+\epsilon}}>\alpha \sqrt{\epsilon}\right)\leq 
\exp\left(-\frac{\alpha^2 \nu^* \epsilon}{2(\sigma^2+B\alpha\sqrt{\epsilon}/3\sqrt{\mu+\epsilon})}\right),\]
where $\sigma^2=\expectwithbrackets{}{(\xi_i/\sqrt{\mu+\epsilon})^2}\le c\mu/(\mu+\epsilon)$.  The lemma then follows from observing that
$c\mu/(\mu+\epsilon)\le c$ (as $\mu\ge 0$ and $\epsilon> 0$) and $B\alpha\sqrt{\epsilon}/3\sqrt{\mu+\epsilon}\le B/3$ 
(as $\mu\ge 0$, $\epsilon\ge 0$ and $0<\alpha\le 1$).
\end{proof}

Lemma \ref{lemma:ratiosingle} can also be extended to families of functions as follows.
\begin{lemma}\label{lemma:ratiofamilyclass}
Let $\mathcal{G}$ be a set of functions on $\mathcal{Z}$ and $c>0$ such that, for each $g\in\mathcal{G},$ $\expectwithbrackets{}{g}\ge 0$, 
$\expectwithbrackets{}{g^2}\leq c\expectwithbrackets{}{g}$ and $|g-\expectwithbrackets{}{g}|\leq B$ almost everywhere. 
Let $\left(w_i\right)_{i=1}^{\SizeEdge}$ be any FMN weight vector.
Then for every $\epsilon>0$ and $0<\alpha\le 1,$ we have
% $${\Pr} \left(\sup_{g\in \mathcal{G}} \frac{\expectwithbrackets{}{g}-\frac{1}{\nu^*}\sum_{i=1}^{\SizeEdge} w_i g(z_i)}{\sqrt{\expectwithbrackets{}{g_j}+\epsilon}} \ge 4 \alpha \sqrt{\epsilon}\right)
$${\Pr} \left(\sup_{g\in \mathcal{G}} \frac{\expectwithbrackets{}{g}-\frac{1}{\nu^*}\sum_{i=1}^{\SizeEdge} w_i g(z_i)}{\sqrt{\expectwithbrackets{}{g}+\epsilon}} \ge 4 \alpha \sqrt{\epsilon}\right)
\leq {N}(\mathcal{G},\alpha \epsilon)\exp\left(-\frac{\alpha^2 \nu^* \epsilon}{2c+\frac{2}{3}B}\right).$$
\end{lemma}
\begin{proof}
Let $\{g_j\}_{j=1}^J \subset {\cal G}$ with $J=N({\cal G}, \alpha\epsilon)$ be such that $\cal G$ is covered by balls in $\cal C(Z)$ 
centered at $g_j$ with radius $\alpha\epsilon$.

Applying Lemma \ref{lemma:ratiosingle} to $\xi_i = g_j(z_i)$ for each $j$, we have
$$\Pr \left(\frac{\expectwithbrackets{}{g_j}-\frac{1}{\nu^*}\sum_{i=1}^{\SizeEdge} w_i g_j(z_i) }{\sqrt{\expectwithbrackets{}{g_j}+\epsilon}}\ge\alpha \sqrt{\epsilon}\right)\leq 
\exp\left(-\frac{\alpha^2 \nu^* \epsilon}{2c+\frac{2}{3}B}\right).$$
For each $g\in {\cal G}$, there is some $j$ such that $||g-g_j||_{\cal C(Z)}\le \alpha\epsilon.$
Then $|\frac{1}{\nu^*}\sum_{i=1}^{\SizeEdge} g(z_i) - \frac{1}{\nu^*}\sum_{i=1}^{\SizeEdge} w_i g_j(z_i)|$ and $|\expectwithbrackets{}{g}-\expectwithbrackets{}{g_j}|$ are both bounded by $\alpha\epsilon$.
Hence, as $\frac{\sqrt{\epsilon}}{\sqrt{\epsilon+\expectwithbrackets{}{g}}}\le 1$,
$$\frac{|\frac{1}{\nu^*}\sum_{i=1}^{\SizeEdge} g(z_i) - \frac{1}{\nu^*}\sum_{i=1}^{\SizeEdge} g_j(z_i)|}{\sqrt{\expectwithbrackets{}{g}+\epsilon}}\le\alpha\sqrt{\epsilon}$$
and
$$\frac{|\expectwithbrackets{}{g}-\expectwithbrackets{}{g_j}|}{\sqrt{\expectwithbrackets{}{g}+\epsilon}}\le\alpha\sqrt{\epsilon}.$$
The latter implies that 
\begin{eqnarray*}
 \expectwithbrackets{}{g_j}+\epsilon &=& \expectwithbrackets{}{g_j} -\expectwithbrackets{}{g}+\expectwithbrackets{}{g}+\epsilon 
\le \alpha\sqrt{\epsilon}\sqrt{\expectwithbrackets{}{g}+\epsilon}+(\expectwithbrackets{}{g}+\epsilon)\\
&\le&\sqrt{\epsilon}\sqrt{\expectwithbrackets{}{g}+\epsilon}+(\expectwithbrackets{}{g}+\epsilon)\le 2(\expectwithbrackets{}{g}+\epsilon).
\end{eqnarray*}
It follows that $\sqrt{\expectwithbrackets{}{g_j}+\epsilon}\le 2\sqrt{\expectwithbrackets{}{g}+\epsilon}.$ 
We have thus seen that $\frac{\expectwithbrackets{}{g}-\frac{1}{\nu^*}\sum_{i=1}^{\SizeEdge} g(z_i)}{\sqrt{\expectwithbrackets{}{g}+\epsilon}}\ge 4\alpha\sqrt{\epsilon}$ implies
$\frac{\expectwithbrackets{}{g_j}-\frac{1}{\nu^*}\sum_{i=1}^{\SizeEdge} w_i g_j(z_i)}{\sqrt{\expectwithbrackets{}{g}+\epsilon}}\ge 2\alpha\sqrt{\epsilon}$ and hence
$\frac{\expectwithbrackets{}{g_j}-\frac{1}{\nu^*}\sum_{i=1}^{\SizeEdge} w_i g_j(z_i)}{\sqrt{\expectwithbrackets{}{g_j}+\epsilon}}\ge \alpha\sqrt{\epsilon}$.
Therefore,
$$\Pr \left(\sup_{g\in \mathcal{G}} \frac{\expectwithbrackets{}{g}-\frac{1}{\nu^*}\sum_{i=1}^{\SizeEdge} w_i g(z_i)}{\sqrt{\expectwithbrackets{}{g}+\epsilon}} \ge 4 
\alpha \sqrt{\epsilon}\right)\le \sum_{j=1}^J \Pr\left(\frac{\expectwithbrackets{}{g_j}-\frac{1}{\nu^*}\sum_{i=1}^{\SizeEdge} w_i g_j(z_i)}
{\sqrt{\expectwithbrackets{}{g_j}+\epsilon}}\ge \alpha\sqrt{\epsilon}\right)$$
which is bouned by $J\cdotp\exp\left(-\frac{\alpha^2\nu^*\epsilon}{2c+\frac{2}{3}B}\right).$
\end{proof}

Let $\mathcal{L}_\rho^2(\spacefull)$ be a Banach space with the norm
$\|f\|_{{\mathcal{L}}_\rho^2({\spacefull})}=\left(\int_\spacefull|f(x)^2|\probxc(x) \hbox{d}x\right)^\frac{1}{2}.$
where $\probxc(x)=\prod_{i=1}^k x^{(i)}$.
We define the error in $\cal H$ of a function $f\in {\cal H}$,
$${\cal E}_{\cal H}(f) = {\cal E}(f)- {\cal E}(f_{\rho,{\cal H}})$$
which is always nonnegative. 

\begin{lemma}\label{lemma:convex}
Let $\mathcal{H}$ be a convex subset of $\mathcal{C}(\spacefull)$ such that $f_{\rho,{\cal H}}$ exists. 
Then $f_{\rho,{\cal H}}$ is unique as an element in $\mathcal{L}_\rho^2(\spacefull)$ and for all 
$f\in \mathcal{H},$ $$\int_{\spacefull}(f_{\rho,{\cal H}}(x)-f(x))^2 \probxc(x) dx \leq \mathcal{E}_\mathcal{H}(f).$$
In particular, if $\probxc(x)$ is not degenerate then $f_{\rho,{\cal H}}$ is unique in $\mathcal{H}.$
\end{lemma}
\begin{proof}
The proof can be found in \citep{cucker07} (Lemma 3.16). 
\end{proof}

\noindent\textbf{Proof of Theorem \ref{thm:nu-bound}}
For every function $f\in {\cal H}$,
we define a function $$g_f(x,y)=(f({x})-y)^2-(f_{\rho,{\cal H}}({x})-y)^2.$$
We define $\cal G$ as the set of all functions $g_f$ with $f\in {\cal H}$. 
For any function $g_f\in\mathcal{G},$ we have 
\begin{equation}\label{eq:identity}
\expectwithbrackets{z\sim\probl}{g_f}=\mathcal{E}_{\mathcal{H}}(f)\ge 0.
\end{equation}
We first show that the two preconditions of Lemma \ref{lemma:ratiofamilyclass} are true (for $B=2M^2$ and $c=4M^2$):
\begin{enumerate}
 \item $|g_f-\expectwithbrackets{z\sim\probl}{g_f}|\leq 2M^2$
 \item $\expectwithbrackets{z\sim\probl}{g_f^2}\leq 4M^2\expectwithbrackets{z\sim\probl}{g_f}$.
\end{enumerate}
First, since $\mathcal{H}$ is $M$-bounded,
we have that $-M^2\leq g_f({z})\leq M^2$ holds almost everywhere.
It follows that $|g_f-\expectwithbrackets{z\sim\probl}{g_f}|\le 2M^2$ holds almost everywhere. This is the first precondition above.
Second, one can easily see that
$$g_f(z)=(f({x})-f_{\rho,{\cal H}}(x))[(f({x})-y)+(f_{\rho,{\cal H}}(x)-y)].$$ 
It follows that $|g_f(z)|\leq 2M|f(x)-f_{\rho,{\cal H}}(x)|$ holds almost everywhere. 
Then, $\expectwithbrackets{z\sim\probl}{g_f^2}\leq 
4M^2\expectwithbrackets{x\sim\probxc}{(f(x)-f_{\rho,{\cal H}}(x))^2} = 4M^2\int_\spacefull\left(f(x)-f_{\rho,{\cal H}}(x)\right)^2 \probxc(x) \hbox{d}x$.
Together with Lemma \ref{lemma:convex} 
this implies that $\expectwithbrackets{z\sim\probl}{g_f^2}\leq 4M^2\mathcal{E}_{\mathcal{H}}(f)=c\expectwithbrackets{z\sim\probl}{g_f}$
with $c=4M^2.$ 
Hence, all the conditions of Lemma \ref{lemma:ratiofamilyclass} hold and we get that for every $\epsilon>0$ and $0<\alpha\leq 1,$
\begin{equation}
\label{eq:lemratiofamilyclass.applied} {\Pr} \left(\sup_{g\in \mathcal{G}} \frac{\expectwithbrackets{}{g}-\frac{1}{\nu^*}\sum_{i=1}^{\SizeEdge} w_i g(z_i)}{\sqrt{\expectwithbrackets{}{g_j}+\epsilon}} \ge 4 \alpha \sqrt{\epsilon}\right)
\leq {N}(\mathcal{G},\alpha \epsilon)\exp\left(-\frac{\alpha^2 \nu^* \epsilon}{2.4M^2+\frac{2}{3}2M^2}\right).
\end{equation}
Remind from Equation \eqref{eq:identity} that $\expectwithbrackets{}{g_f}=\mathcal{E}_{\cal H}(f)$.
We also define
\[{\cal E}_{\mathcal{H},{{Z}_{\nu^*}}}(f)=\frac{1}{\nu^*}\sum_{i=1}^{\SizeEdge} w_i g_f(z_i) =  \frac{1}{\nu^*}\sum_{i=1}^{\SizeEdge} w_i (f({x})-y)^2- \frac{1}{\nu^*}\sum_{i=1}^{\SizeEdge} w_i (f_{\rho,{\cal H}}({x})-y)^2\] 
Furthermore, we take $\alpha=\sqrt{2}/8$.
Substituting all these into Inequality\ \eqref{eq:lemratiofamilyclass.applied} we get 
\[\forall \epsilon>0, \Pr\left(\sup_{f\in\mathcal{H}}\frac{\mathcal{E}_{\cal H}(f)-\mathcal{E}_{\mathcal{H},{{Z}_{\nu^*}}}(f)}{\sqrt{\mathcal{E}_{S}(f)+\epsilon}}\ge 4\frac{\sqrt{2}}{8} \sqrt{\epsilon}\right) \le {N}\left(\mathcal{G},\frac{\sqrt{2}}{8}\epsilon\right)\exp\left(-\frac{\left(\frac{\sqrt{2}}{8}\right)^2 \nu^* \epsilon}{28M^2/3}\right).\]
As this holds for the supremum over $f$, it also holds for $f=f_{{Z}_{\nu^*},{\cal H}}$:
\[\forall \epsilon>0, \Pr\left(\frac{\mathcal{E}_{\cal H}(f_{{Z}_{\nu^*},{\cal H}})-\mathcal{E}_{\mathcal{H},{{Z}_{\nu^*}}}(f_{{Z}_{\nu^*},{\cal H}})}{\sqrt{\mathcal{E}_{\cal H}(f_{{Z}_{\nu^*},{\cal H}})+\epsilon}} \ge \sqrt{\frac{\epsilon}{2}}\right) 
\le {N}\left(\mathcal{G},\frac{\sqrt{2}}{8}\epsilon\right)\exp\left(-\frac{ \nu^* \epsilon}{896M^2/3}\right).\]
The definition of $f_{{Z}_{\nu^*},{\cal H}}$ tells us that ${\cal E}_S(Z_{\nu^*}) = \mathcal{E}_{\cal H}(f_{{Z}_{\nu^*},{\cal H}})$ and $\mathcal{E}_{\mathcal{H},{{Z}_{\nu^*}}}(f_{{Z}_{\nu^*},{\cal H}})\le 0.$
It follows that (we also upper-bound $896/3$ by $300$)
\[\forall \epsilon>0, \Pr\left(\frac{{\cal E}_S(Z_{\nu^*})}{\sqrt{{\cal E}_S(Z_{\nu^*})+\epsilon}}\ge \sqrt{\frac{\epsilon}{2}}\right) \le {N}\left(\mathcal{G},\frac{\sqrt{2}}{8}\epsilon\right)\exp\left(-\frac{ \nu^* \epsilon}{300M^2}\right).\]
It is easy to see that ${\cal E}_S(Z_{\nu^*})\ge \epsilon$ implies $\frac{{\cal E}_S(Z_{\nu^*})}{\sqrt{{\cal E}_S(Z_{\nu^*})+\epsilon}}\ge \sqrt{\frac{\epsilon}{2}}$, so
\[\forall \epsilon>0, \Pr\left({\cal E}_S(Z_{\nu^*})\ge \epsilon\right) \le {N}\left(\mathcal{G},\frac{\sqrt{2}}{8}\epsilon\right)\exp\left(-\frac{ \nu^* \epsilon}{300M^2}\right).\]
Finally, the inequality $\|g_{f_1}-g_{f_2}\|_{\mathcal{C}({\cal Z})}= \|f_1({x})-f_2({x})[(f_1({x})-y)+(f_2({x})-y)]\|_{\mathcal{C}(\mathcal{Z})}\leq 2M\|f_1-f_2\|_{\mathcal{C}(\spacefull)},$
tells us that
$${N}(\mathcal{G},\frac{\sqrt{2}\epsilon}{8})\leq {N}(\mathcal{H},\frac{\sqrt{2}\epsilon}{16M})\le {N}(\mathcal{H},\frac{\epsilon}{12M}).$$
This completes our proof. 
\eof

\end{document}